\documentclass[dvipsnames]{article} %
\usepackage{iclr2026_conference,times}

\usepackage{amsmath,amsfonts,bm}

\def\eqref#1{equation~\ref{#1}}

\def\1{\bm{1}}

\DeclareMathAlphabet{\mathsfit}{\encodingdefault}{\sfdefault}{m}{sl}
\SetMathAlphabet{\mathsfit}{bold}{\encodingdefault}{\sfdefault}{bx}{n}

\newcommand{\E}{\mathbb{E}}

\newcommand{\Var}{\mathrm{Var}}

\newcommand{\Cov}{\mathrm{Cov}}

\usepackage{wrapfig}
\usepackage{natbib}
\usepackage{xcolor}
\usepackage{hyperref}
\usepackage{url}
\usepackage{graphicx}
\usepackage{amsthm}
\usepackage{amssymb}
\usepackage{tcolorbox}
\usepackage{enumitem}
\usepackage{subcaption}
\usepackage{dsfont}
\usepackage{booktabs}
\usepackage{listings}
\usepackage{makecell} %
\usepackage{subcaption}
\usepackage{caption}
\usepackage{siunitx}
\usepackage[normalem]{ulem}
\usepackage{textcomp} 

\usepackage{xcolor}
\definecolor{codebg}{RGB}{248,248,248}
\usepackage[cache=false]{minted} %
\usepackage{booktabs}
\usepackage{siunitx}

\renewcommand{\cite}{\citep}
\title{Aligning Diffusion Language Models via Unpaired Preference Optimization}

\makeatletter
\renewcommand*{\thefootnote}{\fnsymbol{footnote}} % 1=*, 2=†, 3=‡, 4=§
\makeatother

\author{\footnotesize \normalfont
  \begin{tabular}[t]{@{}l@{}}
  \textbf{Vaibhav Jindal}\textsuperscript{1\,\textdagger}\quad
  \textbf{Hejian Sang}\textsuperscript{1}\quad
  \textbf{Chun-Mao Lai}\textsuperscript{3}\quad
  \textbf{Yanning Chen}\textsuperscript{1}\quad
  \textbf{Zhipeng Wang}\textsuperscript{1}
  \end{tabular}
  \\[-1pt]
  \footnotesize \textsuperscript{1} LinkedIn Corporation, CA, USA \quad
  \footnotesize \textsuperscript{3} University of California San Diego, CA, USA
}

\iclrfinalcopy
\begin{document}

\maketitle

% --- after \maketitle ---
\begingroup
\renewcommand\thefootnote{\fnsymbol{footnote}}
\setcounter{footnote}{0}%
\footnotetext[2]{\mbox{Corresponding author: Vaibhav Jindal \textless{}\texttt{vjindal@linkedin.com}\textgreater{}.}}
\endgroup

\begin{abstract}
  Diffusion language models (dLLMs) are an emerging alternative to autoregressive (AR) generators, but aligning them to human preferences is challenging because sequence log-likelihoods are intractable and pairwise preference data are costly to collect. We introduce ELBO-KTO, which combines an ELBO surrogate for diffusion log-likelihoods with a prospect-theoretic, unpaired preference objective (Kahneman–Tversky Optimization, KTO). We analyze the bias and variance induced by the ELBO substitution and employ variance-reduction practices that stabilize gradients during training. Applied to LLaDA-8B-Instruct, ELBO-KTO yields \textbf{65.9\%} and \textbf{62.3\%} adjusted win rates on kto-mix-14k and UltraFeedback-Binary, respectively, versus the base model under an automatic LLM judge. Across downstream tasks, including GSM8K, MMLU, and additional reasoning/knowledge benchmarks, ELBO-KTO trained on UltraFeedback-Binary performs on par with or better than the base model under identical decoding. This establishes unpaired preference optimization as a viable alternative to pairwise alignment in diffusion LLMs. We release our implementation at \href{https://github.com/vaibhavjindal/elbo-kto}{https://github.com/vaibhavjindal/elbo-kto}.
\end{abstract}

\section{INTRODUCTION}

Aligning large language models (LLMs) with human preferences is central to building helpful systems. In autoregressive LLMs, methods such as Reinforcement learning
from human feedback (RLHF) \citep{christiano2017deep, ouyang2022training, stiennon2020learning, schulman2017proximal} and Direct Preference Optimization (DPO) \citep{rafailov2023direct} and its variants \citep{meng2024simpo,yuan2023rrhf, azar2024general} are standard and rely on curated \emph{paired} comparisons for each prompt. Diffusion-style large language models (dLLMs) extend text generation beyond left-to-right decoding by iteratively refining a sequence in parallel \citep{ye2025dream, ou2024your, lou2310discrete, nie2025largelanguagediffusionmodels, khanna2025mercury}  which makes sentence-level likelihoods intractable and complicates preference optimization.

% For dLLMs, recent work replaced log-likelihoods with an Evidence Lower Bound (ELBO) surrogate and introduced variance reduction to stabilize training, but still requires paired preferences \citep{zhu2025llada15variancereducedpreference}. This paired assumption raises cost and coverage barriers when feedback is plentiful but unpaired (binary “good/bad”), where it is a regime common in user thumbs-up/downs and safety filters. We take a complementary route and instantiate \emph{Kahneman–Tversky Optimization (KTO)} \citep{ethayarajh2024ktomodelalignmentprospect}, \emph{unpaired} preference objective—for dLLMs. Our recipe plugs ELBO estimates into KTO’s value formulation by using an ELBO ratio (policy vs.\ reference) and an ELBO baseline, and adopts simple variance reduction practices to control estimator noise.

Recent approaches to dLLMs have replaced log-likelihoods with ELBO-based surrogates and incorporated variance reduction techniques to stabilize training, but remain constrained by the assumption of paired preference data \citep{zhu2025llada15variancereducedpreference}. This assumption limits scalability and applicability in practical settings where feedback is abundant but inherently unpaired—as in binary “good/bad” user ratings or safety filter signals.

In this work, we take a fundamentally different direction by proposing ELBO-KTO as a framework to train dLLMs directly from unpaired preference signals. Our approach integrates ELBO estimates into KTO’s value computation via an ELBO margin (policy vs.\ reference) and a baseline for variance control. This formulation unlocks learning from a broader class of real-world “good/bad” feedback signals while preserving training stability through lightweight variance reduction techniques.

The main contributions are summarized as follows:
\begin{itemize}
    \item We present ELBO-KTO, a principled preference optimization framework that enables alignment of diffusion language models (dLLMs) from unpaired feedback. 
    \item We provide a theoretical analysis of the bias and variance tradeoff introduced by ELBO substitution, and prove that our estimator enjoys bounded bias with controlled variance.
    \item We empirically validate our proposed method ELBO-KTO on kto-mix-14k \citep{trl_kto_mix_14k} and UltraFeedback-Binary \citep{cui2023ultrafeedback}, where it achieves \textbf{65.9\%} and \textbf{62.3\%} adjusted win rates over the base model, demonstrating consistent improvements. Extensive analysis confirms that our method provides a simple yet effective way to align dLLMs with stronger models. On downstream reasoning/knowledge tasks such as GSM8K and MMLU, ELBO-KTO trained on UltraFeedback-Binary performs on par with or better than the LLaDA-8B-Instruct.

\end{itemize}

\section{PRELIMINARIES}

\subsection{Diffusion Large Language Models and LLaDA}

LLaDA \citep{nie2025largelanguagediffusionmodels} is an 8B-parameter masked diffusion model (MDM) pretrained on 2.3T tokens (LLaDA-8B-Base) and fine-tuned on 4.5M instruction pairs (LLaDA-8B-Instruct). LLaDA achieves performance competitive with leading autoregressive LLMs, demonstrating strong scalability, in-context learning, and instruction-following capabilities. In this work, we adopt LLaDA-8B-Instruct as our base model for studying alignment in diffusion LLMs using the KTO algorithm.

\subsection{VRPO: Variance Reduced Preference Optimization for dLLMs}

While Direct Preference Optimization (DPO) operates on sentence-level log-likelihoods, such quantities are intractable in masked diffusion LMs. Variance-Reduced Preference Optimization (VRPO) addresses this by replacing log-likelihoods with ELBO surrogates that decompose over diffusion steps \citep{zhu2025llada15variancereducedpreference}.

\paragraph{ELBO.}
Given a prompt $x$, and a corresponding full response $y$, we draw a noised intermediate $y_t$ from the forward process $q(y_t \mid t,y,x)$ at a randomly sampled diffusion timestep $t \sim \mathcal{U}[0,1]$ and evaluate a per-step loss $\ell_\pi$. Averaging over $t$ and $y_t$ yields the ELBO
\begin{equation}
\mathcal{B}_{\pi}(y \mid x)
= \mathbb{E}_{t\sim\mathcal{U}[0,1]}\ \mathbb{E}_{y_t \sim q(\cdot \mid t,y,x)}\!\big[\ell_\pi(y_t, t, y \mid x)\big],
\label{eq:vrpo-elbo}
\end{equation}
where $\ell_\pi$ is the per-step loss of the mask prediction model as defined in Appendix A. $\mathcal{B}_{\pi}$ is a lower-bound of $\log \pi_\theta(y \mid x)$ \citep{lou2310discrete,ou2024your}.

\paragraph{Monte Carlo Estimator.}
Sampling $n_t$ time steps, $\{t_j\}_{j=1}^{n_t} \sim \mathcal{U}[0,1]$, and $n_{y_t}$ draws per timestep, $\{y^{(i)}_{t_j}\}_{i=1}^{n_{y_t}} \sim q(y_{t_j}\mid t_j,y,x)$,
\begin{equation}
\widehat{\mathcal{B}}_{\pi}(y\mid x)
= \frac{1}{n_t}\sum_{j=1}^{n_t}\frac{1}{n_{y_t}}\sum_{i=1}^{n_{y_t}} \ell_\pi\!\big(y^{(i)}_{t_j}, t_j, y\mid x\big)
\label{eq:B-hat}
\end{equation}
is the Monte Carlo Estimator of $\mathcal{B}_{\pi}$.

VRPO introduces techniques to reduce the variance of this estimator by (i) increasing the sampling budget, (ii) allocating more samples across diffusion steps by increasing $n_t$, and (iii) using common random numbers shared between policy and reference. These techniques help stabilize training and reduce the variance of loss and gradient.

\paragraph{Reference-Adjusted ELBO margin.}
Given a desirable completion $y_w$ and a undesirable completion $y_l$, VRPO compares their \emph{policy-vs-reference} ELBO differences:
\begin{equation}
\begin{aligned}
\widehat{m}_\theta(x;y_w,y_l)
&= \Big[\widehat{\mathcal{B}}_{\pi_\theta}(y_w \mid x)-\widehat{\mathcal{B}}_{\pi_\theta}(y_l \mid x)\Big] \\
&\quad - \Big[\widehat{\mathcal{B}}_{\pi_\mathrm{ref}}(y_w \mid x)-\widehat{\mathcal{B}}_{\pi_\mathrm{ref}}(y_l \mid x)\Big].
\end{aligned}
\label{eq:vrpo-margin}
\end{equation}

\paragraph{VRPO Loss.}

Finally, the DPO-style logistic loss over the margin is
\begin{equation}\label{eq:vrpo-loss}
\mathcal{L}_{\mathrm{VRPO}}(\theta)
= \mathbb{E}_{(x,y_w,y_l)\sim\mathcal{D}}
\bigl[-\log \sigma\bigl(\beta\,\widehat{m}_\theta(x;y_w,y_l)\bigr)\bigr]
\end{equation}
% punctuation belongs outside the display
\noindent where $\sigma(\cdot)$ is the sigmoid and $\beta>0$ is a temperature. This loss encourages the policy to assign a relatively higher probability mass to $y_w$ than to $y_l$ under ELBO surrogates, thus teaching the preference to the model.

\subsection{Kahneman–Tversky Optimization (KTO)}

KTO \citep{ethayarajh2024ktomodelalignmentprospect} is a preference optimization algorithm inspired by prospect theory in behavioral economics, which models human perception of gains and losses as asymmetric. Unlike DPO, which relies on paired preference data, KTO can learn directly from unpaired binary signals (“desirable” vs.~“undesirable”), making it well-suited for settings with limited or imbalanced feedback.

Let $\pi_\theta$ be the policy, $\pi_{\mathrm{ref}}$ be a frozen reference, and $\sigma(\cdot)$ the sigmoid function. KTO defines the \emph{reference–adjusted reward} as
\begin{equation}
    r_\theta(x,y) \;=\; \log\frac{\pi_\theta(y\mid x)}{\pi_{\mathrm{ref}}(y\mid x)} .
\end{equation}
This reward is mapped through a prospect-theoretic \emph{value function} with separate loss-aversion weights for desirable and undesirable examples:
\begin{equation}
    v(x,y) =
\begin{cases}
\lambda_D\sigma(\beta \,(r_\theta(x,y)-z_0)), & y \sim y_{\text{desirable}}|x, \\[3pt]
\lambda_U\sigma(\beta \,(z_0-r_\theta(x,y))), & y \sim y_{\text{undesirable}}|x .
\end{cases}
\label{eq:kto-value}
\end{equation}
The baseline
\[
z_0=\text{KL}(\pi_{\theta}(\cdot|x)\,\|\,\pi_{\mathrm{ref}}(\cdot|x))
\]
acts as a reference point to control loss saturation. Here, $\beta>0$ controls risk aversion, while $\lambda_D$ and $\lambda_U$ govern asymmetric loss aversion for desirable and undesirable samples.

Assigning $\lambda_y=\lambda_D$ for desirable and $\lambda_y=\lambda_U$ for undesirable samples, the KTO objective is
\begin{equation}
    \mathcal{L}_{\mathrm{KTO}}(\theta) = \mathbb{E}_{(x,y)\sim\mathcal{D}}[\lambda_y - v(x,y)] .
\end{equation}
Intuitively, desirable examples increase utility when the policy’s log-likelihood exceeds that of reference, while undesirable ones increase utility when it falls below the reference. The term $\lambda_y$ ensures non-negativity. In practice, gradients are only taken with respect to $\pi_\theta$, while $\pi_{\mathrm{ref}}$ and $z_0$ are treated as stop-gradient terms, and we don't backpropagte through them.

\section{METHOD}
\label{sec:kto-dllm}
\subsection{ELBO-KTO for dLLMs}
KTO scores a response by the log-likelihood ratio
$\log\frac{\pi_\theta(y\mid x)}{\pi_{\mathrm{ref}}(y\mid x)}$,
which is intractable for diffusion LMs. To solve for this, we replace each
log-likelihood with a Monte Carlo ELBO lower bound and work with the
\emph{ELBO margin}
\begin{equation}
\widehat{r}_\theta(x,y)\;=\;\widehat{\mathcal{B}}_{\pi_\theta}(y\mid x)\;-\;\widehat{\mathcal{B}}_{\pi_\mathrm{ref}}(y\mid x),
\label{eq:elbo-margin}
\end{equation}
where $\widehat{\mathcal{B}}_{\pi}$ is defined in (\ref{eq:B-hat}) as 
\begin{equation*}
    \widehat{\mathcal{B}}_{\pi}(y\mid x)
= \frac{1}{n_t}\sum_{j=1}^{n_t}\frac{1}{n_{y_t}}\sum_{i=1}^{n_{y_t}} \ell_\pi\!\big(y^{(i)}_{t_j}, t_j, y\mid x\big),
\end{equation*}
and $\ell_\pi$ represents the per-step loss of the mask prediction model, whose exact formulation is defined in Appendix A. 

\subsection{Global Per-Batch Baseline for Variance Control}

Classical KTO introduces a per-prompt reference point
$$z_0(x)=\mathbb{E}_{y'\sim \pi_\theta(\cdot\mid x)}[\log\frac{\pi_\theta(y'\mid x)}{\pi_{\mathrm{ref}}(y'\mid x)}],$$ i.e. a KL term. This reference is intractable for dLLMs and estimating it with Monte Carlo methods is prohibitively expensive. 
Thus, for computational efficiency and stability, we use a \emph{single scalar
global baseline} computed per mini-batch $S$
\begin{equation}
\hat{b}_0(S)\;=\;\frac{1}{m}\sum_{i=1}^m \widehat{r}_\theta(x_i,y_i),
\qquad S=\{(x_i,y_i)\}_{i=1}^m,
\label{eq:global-baseline}
\end{equation}
and treat $\hat{b}_0(S)$ as stop-gradient, i.e., we do not backpropagate through it. This is a standard control variate which recenters the scores entering the KTO sigmoid function. It reduces the gradient variance without any
additional ELBO evaluations or KL baseline calculations using mismatched pairs as done in \citet{ethayarajh2024ktomodelalignmentprospect}.

\subsection{Instantiated Loss for Diffusion Language Models}

Let $s_i\in\{+1,-1\}$ encode desirable/undesirable and
$g(u)=\sigma(\beta u)$ be the logistic link, where $\beta$ controls the KTO risk aversion. We instantiate the general KTO
Loss for a mini-batch $S$ with the MC ELBO margin
(\ref{eq:elbo-margin}) and the global baseline (\ref{eq:global-baseline}):
\begin{equation}
\widehat{L}(S)
\;=\;\frac{1}{m}\sum_{i=1}^m
\lambda_{s_i}\;\Big(1- g(s_i[\widehat{r}_\theta(x_i,y_i)-\hat{b}_0(S)])\Big ),
\label{eq:gb-kto-elbo}
\end{equation}
where only the $\widehat{\mathcal{B}}_{\pi_\theta}$ term backpropagates; the reference term, $\widehat{\mathcal{B}}_{\pi_\mathrm{ref}}$ and
$\hat{b}_0(S)$ are treated as constants. This technique centers the scores
around the batch mean to improve stability and efficiency while preserving the ELBO-difference
structure needed for diffusion LLMs.

Subtracting a constant baseline minimizes the variance of the centered scores;
with logistic $g$, this keeps $s_i[\widehat{r}_\theta- \hat{b}_0]$ near the
high slope region, avoiding loss saturation and reducing the variance of the gradient.
Using the batch mean (\ref{eq:global-baseline}) gives this benefit without additional compute. A detailed bias-variance discussion for the instantiated loss and the gradient bounds
for (\ref{eq:gb-kto-elbo}) is provided in the Section \ref{sec:theory}.

\subsection{Variance Reduction}

We show theoretically in Section \ref{sec:theory} that VRPO-style ELBO variance controls work for ELBO-KTO as well and we use these techniques in our implementation. Particularly, we allocate MC budget across
diffusion steps and share random draws between the policy and
the reference to induce positive covariance between their $\widehat{\mathcal{B}}$ estimators.

\subsection{Efficiency Considerations}

\paragraph{Half the Forward-Backward Passes.}
Compared to pairwise DPO on dLLMs, ELBO-KTO uses one policy forward/backward and one reference forward per example (vs. two of each), yielding $\sim2\times$ lower activations or $\sim2\times$ effective batch size at the same memory.

\paragraph{Cheaper Data.}
Binary, unpaired labels avoid pair construction. KTO handles the data imbalance via asymmetric loss aversion for desired and undesired samples.

\paragraph{Zero Compute Baseline.}
The global baseline (\ref{eq:global-baseline}) is computed from the same batch used for training; no mismatched pairs or additional MC evaluations are required to estimate the KL baseline.

\newtheorem{theorem}{Theorem}
\newtheorem{lemma}{Lemma}
\newtheorem{definition}{Definition}
\newcommand{\ED}{\mathbb{E}_{\mathcal{D}}}      % expectation over dataset sampling (minibatches)
\newcommand{\Emc}{\mathbb{E}_{\mathrm{MC}}}      % expectation over MC randomness
\newcommand{\EMC}{\mathbb{E}_{\mathrm{MC}}} 
\newcommand{\Varmc}{\mathrm{Var}_{\mathrm{MC}}}  % variance over MC randomness
% \newcommand{\Cov}{\mathrm{Cov}}

% Restated theorem (unnumbered) with italic body, using original label's number
\newenvironment{restatedtheorem}[1]{%
  \par\noindent\textbf{Theorem~\ref{#1} (Restated).}\ \itshape
}{%
  \par\normalfont
}

% (Optional) siblings for lemmas/props
\newenvironment{restatedlemma}[1]{%
  \par\noindent\textbf{Lemma~\ref{#1} (Restated).}\ \itshape
}{%
  \par\normalfont
}

%========================================================
\section{THEORETICAL ANALYSIS}
\label{sec:theory}
%========================================================
We present a theoretical analysis of ELBO-KTO and give bounds for bias and variance of loss and gradient. Using these bounds, we further motivate the use of VRPO-style variance reduction techniques to stabilize training. We also demonstrate the optimality of the global per-batch baseline $\hat b_0$ as a constant baseline for a mini-batch. For clarity, we only present the main theorems in this section and defer a detailed analysis with proofs to Appendix B.

%========================================
\subsection{Setup and Notation}
\label{sec:setup-notation}
%========================================

\paragraph{Dataset and Minibatch.}
Let the finite dataset be $\mathcal{D}_N=\{(x_n,y_n)\}_{n=1}^N$, where $x_n$ is prompt and $y_n$ is response.
A minibatch is $S=\{(x_i,y_i,s_i,\lambda_i)\}_{i=1}^m$, sampled uniformly. Here $s_i\in\{-1,+1\}$ denotes undesirable/desirable,
and $\lambda_i\in\{\lambda_D,\lambda_U\}$ are class weights with
$\lambda_{\max}=\max\{\lambda_D,\lambda_U\}$.

\paragraph{Policies and ELBO Plug-In.}
Let $\pi_\theta(\cdot\mid x)$ be the current policy and
$\pi_{\mathrm{ref}}(\cdot\mid x)$ be the frozen reference.
We use an ELBO surrogate $B_\pi(y\mid x)$ and an unbiased MC estimator
$\widehat B_\pi(y\mid x)$ satisfying
$\Emc[\widehat B_\pi(y\mid x)]=B_\pi(y\mid x)$.

\paragraph{ELBO Margins and Global Baseline.}
Define the ELBO margin as
\begin{equation}
    r_i \;:=\; B_{\pi_\theta}(y_i\mid x_i)-B_{\pi_{\mathrm{ref}}}(y_i\mid x_i),
\end{equation}
and its MC estimate as
\begin{equation}
    \hat r_i \;:=\; \widehat B_{\pi_\theta}(y_i\mid x_i)-\widehat B_{\pi_{\mathrm{ref}}}(y_i\mid x_i).
\end{equation}
The global mini-batch baseline is
$\hat b_0=\tfrac{1}{m}\sum_{j=1}^m \hat r_j$, and the signed centered margin is
$\hat\delta_i=s_i(\hat r_i-\hat b_0)$.
For gradients, we treat $\hat b_0$ and the reference term as stop-grad.

\paragraph{Objective and Target.}
With logistic link $g(u)=\sigma(\beta u)$, the per-item and batch losses are
\begin{align}
\hat\ell_i \;=\; \lambda_i\big(1-g(\hat\delta_i)\big),\qquad
\hat L(S) \;=\; \tfrac{1}{m}\sum_{i=1}^m \hat\ell_i.
\end{align}
The corresponding noise-free target \emph{global-baseline} replaces the MC estimates by expectations:
\begin{align}
L^{\mathrm{sg}}_{\mathrm{GB}}(S;\theta)
=\tfrac{1}{m}\sum_{i=1}^m \lambda_i\big(1-g\big(s_i(r_i-b_0)\big)\big),
\end{align}
where $b_0=\Emc[\hat b_0] = \tfrac{1}{m}\sum_{j=1}^m r_j$.
\paragraph{Lipschitz Constants.}
We use $L_g$ to denote Lipschitz constant of $g$ and $L_{g'}$ to denote the Lipschitz constant of $g'$. Since $g=\sigma(\beta u)$, the values of these constants can be derived exactly as shown in Appendix B. These values come out to be $L_g = \beta/4$ and $L_{g'}=\beta^2/(6\sqrt{3})$.

\paragraph{MC Design and Correlation Structure.}
Conditioning on the batch $S$, we can write $\Emc[\hat r_i]=r_i$. Assuming the exchangeability of the indices under the MC design, let $v(S):=\Varmc(\widehat r_i)$ and $c(S):=\Cov(\widehat r_i,\widehat r_j)$ for $i\neq j$.

\paragraph{Centered-Margin Variance Aggregator.}
All our bounds depend on the scalar

\begin{equation}
\begin{aligned}
\Psi(S)\;&:=\;\tfrac{1}{m}\sum_{i=1}^m 
\Emc\!\big[\Var_{\mathrm{MC}}(\hat\delta_i)\big] \\
\;&=\;\tfrac{m-1}{m}\,\big( v(S)- c(S)\big),
\label{eq:Psi}
\end{aligned}
\end{equation}
where the identity follows by expanding $\Var(\hat r_i-\hat b_0)$ under the exchangeability assumptions as shown in Appendix B.

%========================================
\subsection{Loss Bias}
\label{sec:loss-bias}
%========================================

\begin{theorem}[Loss bias bound]
\label{thm:loss-bias}
The minibatch loss bias relative to the global-baseline target satisfies
\begin{equation*}
\begin{aligned}
\Big|
\ED\Emc[\hat L(S)]
- \ED[L^{\mathrm{sg}}_{\mathrm{GB}}(S;\theta)]
\Big|
\;\le\; \\
\lambda_{\max} L_g\;
\ED\!\big[\sqrt{\Psi(S)}\big],
\end{aligned}
\end{equation*}
where $\Psi(S)$ is the centered-margin variance
aggregator defined in~\eqref{eq:Psi}.
\end{theorem}

This result shows that the loss bias arises purely from applying the link function $g$ to noisy centered margin $\hat \delta_i$. Shrinking $\Psi(S)$ directly tightens this gap.

%========================================
\subsection{Loss Variance}
\label{sec:loss-var}
%========================================

\begin{theorem}[Loss variance bound]
\label{thm:loss-var}
For the minibatch loss $\hat L(S)=\tfrac{1}{m}\sum_{i=1}^m \lambda_i\big(1-g(\hat\delta_i)\big)$
with $g(u)=\sigma(\beta u)$,
the MC-induced variance satisfies
\begin{align*}
\Var_{\mathrm{MC}}(\hat L(S))
\;\le\;
(\lambda_{\max}L_g)^2\;
\ED\!\big[\Psi(S)\big].
\end{align*}
\end{theorem}

This bound again depends on the centered-margin variance aggregator $\Psi(S)$ and shrinking it will help reduce the loss variance.

%========================================
\subsection{Gradient Bias}
\label{sec:grad-bias}
%========================================

We define the stochastic gradient for a minibatch $S$ as
\begin{equation}
\begin{aligned}
\label{eq:g-hat-s}
\hat G(S)
&=\frac{1}{m}\sum_{i=1}^m \hat a_i \,\nabla_\theta \hat r_i,
\end{aligned}
\end{equation}
where $\hat a_i = -\lambda_i s_i g'(\hat\delta_i)$ and
$\hat\delta_i = s_i(\hat r_i-\hat b_0)$. Also, let
\begin{equation}
\label{eq:g}
    G = \frac{1}{m}\sum_{i=1}^m a_i \,\nabla_\theta r_i,
\end{equation}

where $a_i = -\lambda_i s_i g'(\delta_i)$ and
$\delta_i=s_i(r_i-b_0)$.
We assume unbiased ELBO gradients $\Emc[\nabla_\theta \hat r_i]=\nabla_\theta r_i$.

\begin{theorem}[Gradient bias bound]
\label{thm:grad-bias} Let $\|\cdot\|$ be the $L_2$-norm.
If $\|\nabla_\theta r_i\|\le \bar G$ for all items, then
\begin{equation*}
\begin{aligned}
\bigl\|\ED&\Emc[\hat G(S)] - \ED[G(S)]\bigr\|
\le \\ &\lambda_{\max} L_{g'}\, \ED\!\big[\sqrt{\Psi(S)}\big]\, \bar G + \lambda_{\max} L_g\, \ED\!\big[\sqrt{\bar v_{\nabla}(S)}\big],
\end{aligned}
\end{equation*}

where $\bar v_{\nabla}(S) = \tfrac{1}{m}\sum_{i=1}^m
\Emc\!\big[\|\nabla_\theta \hat r_i - \nabla_\theta r_i\|^2\big]$.
\end{theorem}

This result shows that gradient bias has two contributions:
weight noise and score-gradient noise. Weight noise arises from applying $g$ to noisy centered margin $\hat \delta_i$, and is controlled by
$\Psi(S)$. Score-gradient noise arises from stochastic ELBO gradients
controlled by $\bar v_{\nabla}(S)$, which can be reduced by increasing the MC budget.

%========================================
\subsection{Gradient Variance}
\label{sec:grad-var}
%========================================

We define
\begin{equation*}
    \begin{gathered}
        \tilde G^2(S) = \tfrac{1}{m}\sum_{i=1}^m \Emc\!\big[\|\nabla_\theta \hat r_i\|^2\big], \\
\bar c_{\nabla}(S)=\tfrac{1}{m(m-1)}\sum_{i\ne j}\Emc[\langle \xi_i,\xi_j\rangle],
    \end{gathered}
\end{equation*}
where $\xi_i = \nabla_\theta \hat r_i - \nabla_\theta r_i$, and assume unbiased ELBO gradients $\Emc[\nabla_\theta \hat r_i]=\nabla_\theta r_i$.

\begin{theorem}[Gradient variance]
\label{thm:grad-var}
For mini-batch $S$, let $U(S)=\frac{1}{m}\sum_{i=1}^m a_i\,\nabla_\theta \hat r_i$.
Then
\begin{equation*}
\begin{aligned}
&\Var_{\mathrm{MC}}\!\big(\hat G(S)\big)
\le\; \\
&\qquad \Big(\sqrt{\Var_{\mathrm{MC}}\!\big(U(S)\big)}
\;+\;
\lambda_{\max} L_{g'}\ \sqrt{\Psi(S)}\; \tilde G(S)\Big)^{\!2},
\end{aligned}
\end{equation*} 
where 
\begin{equation*}
\begin{aligned}
\Var_{\mathrm{MC}}(U)
\;\le\;
(\lambda_{\max}L_g)^2
\Big(\frac{\bar v_{\nabla}(S)}{m}
\;+\;\frac{m-1}{m}\ \bar c_{\nabla}(S)\Big).
\end{aligned}
\end{equation*}
\end{theorem}

For the special case of independent per–item MC seeds, 
we can assume that the per–item MC randomness is conditionally independent across items in the batch, so that $\Emc[\langle \xi_i,\xi_j\rangle]=0$ for $i\neq j$.
Then the variance of $U(S)$ collapses to the diagonal:
\begin{equation*}
\begin{aligned}
\Var_{\mathrm{MC}}\!\big(U(S)\big)
&= \frac{1}{m^2}\sum_{i=1}^m a_i^2\,\Emc\!\big[\|\xi_i\|^2\big] \\
&\;\le\;
(\lambda_{\max}L_g)^2\,\frac{\bar v_{\nabla}(S)}{m},
\end{aligned}
\end{equation*}
where $\bar v_{\nabla}(S)=\tfrac{1}{m}\sum_{i=1}^m \Emc\!\big[\|\xi_i\|^2\big]$.
Substituting into Theorem~\ref{thm:grad-var} gives
\begin{equation*}
\begin{aligned}
&\Var_{\mathrm{MC}}\!\big(\hat G(S)\big)
\le \\
&\qquad\Big((\lambda_{\max}L_g)\,\sqrt{\tfrac{\bar v_{\nabla}(S)}{m}}
\;+\;
\lambda_{\max} L_{g'}\ \sqrt{\Psi(S)}\;\tilde G(S)\Big)^{\!2}.
\end{aligned}
\end{equation*}

Similar to gradient bias, gradient variance can also be bounded by two sources, weight noise and score-gradient noise.

%========================================
\subsection{Estimator Design Strategies}
\label{sec:vrpo-design}
%========================================

We now make the dependence on the estimator design explicit.
Throughout this section we assume the exchangeable MC setting of
Section~\ref{sec:setup-notation}, so that $v(S)$ and $c(S)$ are well-defined. Our goal is to strategically decrease $v(S)$ and increase $c(S)$, thereby shrinking the common driver $\Psi(S)$.

%-----------------------
\paragraph{Decreasing $v(S)$.}
Expanding the formula for $v(S)$, we get
\begin{align*}
v(S)
&=\Var_{\mathrm{MC}}(\hat r_i) \\
&= \Var_{\mathrm{MC}} ( \widehat B_{\pi_\theta}(y_i\mid x_i)-\widehat B_{\pi_{\mathrm{ref}}}(y_i\mid x_i)) \\
&= \Var_{\mathrm{MC}}(\widehat B_{\pi_\theta}(y_i\mid x_i)) + \Var_{\mathrm{MC}}(\widehat B_{\pi_{\mathrm{ref}}}(y_i\mid x_i)) \\ &\qquad -2 \Cov (\widehat B_{\pi_\theta}(y_i\mid x_i), \widehat B_{\pi_{\mathrm{ref}}}(y_i\mid x_i))
\label{eq:vS-decomp}
\end{align*}
Using Theorem 2 from \cite{zhu2025llada15variancereducedpreference}, we get that the $\Var_{\mathrm{MC}} (\widehat B)$ terms decrease on (i) Increasing the overall sampling budget, and (ii) Allocating the full budget to timesteps. Additionally, from Theorem 3 of \cite{zhu2025llada15variancereducedpreference}, we get that using antithetic sampling, i.e., using shared random numbers to compute $\widehat B_{\pi_\theta}$ and $\widehat B_{\pi_{\mathrm{ref}}}$, leads to a positive covariance term, thus decreasing $v(S)$.

%-----------------------
\paragraph{Increasing $c(S)$.}
Recall $c(S)=\Cov(\hat r_i,\hat r_j)$ under the exchangeability of indices. Intuitively, sharing random numbers across items in the minibatch $S$ will lead to a positive co-movement of $\hat r_i$ and $\hat r_j$, thereby increasing $c(S)$. This can further reduce $\Psi(S)$ by raising $c(S)$. For now, we do not use this technique in our experiments due to time constraints and leave this for future work.

%========================================
\subsection{Global Baseline Optimality}
\label{sec:baseline-opt}
%========================================

\begin{lemma}[Global Baseline Optimality]
\label{lem:psi-diff}
For any baseline $b$ that is constant across items in a batch and may depend on MC randomness,
\[
\Psi_b(S)-\Psi_{\hat b_0}(S)
\;=\;
\Var_{\mathrm{MC}}\!\big(b-\hat b_0\big)
\;\ge\;0,
\]
where $\hat b_0=\tfrac{1}{m}\sum_j \hat r_j$. Hence every baseline of the form
$b=\hat b_0+K$ with deterministic constant $K$ attains the same minimum value
$\Psi_b(S)=\Psi_{\hat b_0}(S)$. If, in addition, $\Emc[b]=b_0$, then $K=0$
and the unique minimizer is $b=\hat b_0$.
\end{lemma}

This result says that $\hat b_0$ is variance-optimal for all possible values of $b$. It requires no additional compute and is thus the principled default for ELBO-KTO.

\section{EXPERIMENTS}

\subsection{ELBO-KTO Training Recipe}

\paragraph{Datasets.}
We evaluate on two public preference datasets. kto-mix-14k contains $\sim$13.5k training prompts with desirable/undesirable labels and a 1.5k test split (750/750). It is the KTO counterpart of dpo-mix-7k \citep{argilla_dpo_mix_7k}.
UltraFeedback-Binary provides 61.1k train and 2k test pairs; we convert pairs to unpaired labels for ELBO-KTO by treating chosen responses as desirable and rejected responses as undesirable.
All models are trained on the respective train split and evaluated on the held-out test split.

\paragraph{Implementation details.}
We train for one epoch with batch size $8$ using AdamW (weight decay $0.01$, $\beta_1{=}0.9$, $\beta_2{=}0.95$) and a $3\%$ linear warmup followed by cosine decay.
For kto-mix-14k, the peak learning rate is $1{\times}10^{-6}$ and we draw $8$ MC samples per example.
For UltraFeedback-Binary, the peak learning rate is $5{\times}10^{-7}$ with $4$ MC samples per example.
We use the LLaDA-8B-Instruct model as reference $\pi_{\mathrm{ref}}$ and precompute $\widehat{B}_{\pi_\mathrm{ref}}$ once to avoid loading both policy and reference models simultaneously during training. Other experimental details related to evaluation can be found in Appendix C.

\subsection{Overall test-set results}
\label{sec:overall-results}

% in body (not inside a minipage)
\begin{figure*}[!t]           % star = span both columns; !t = try hard for top
  \centering
  \includegraphics[width=0.78\textwidth]{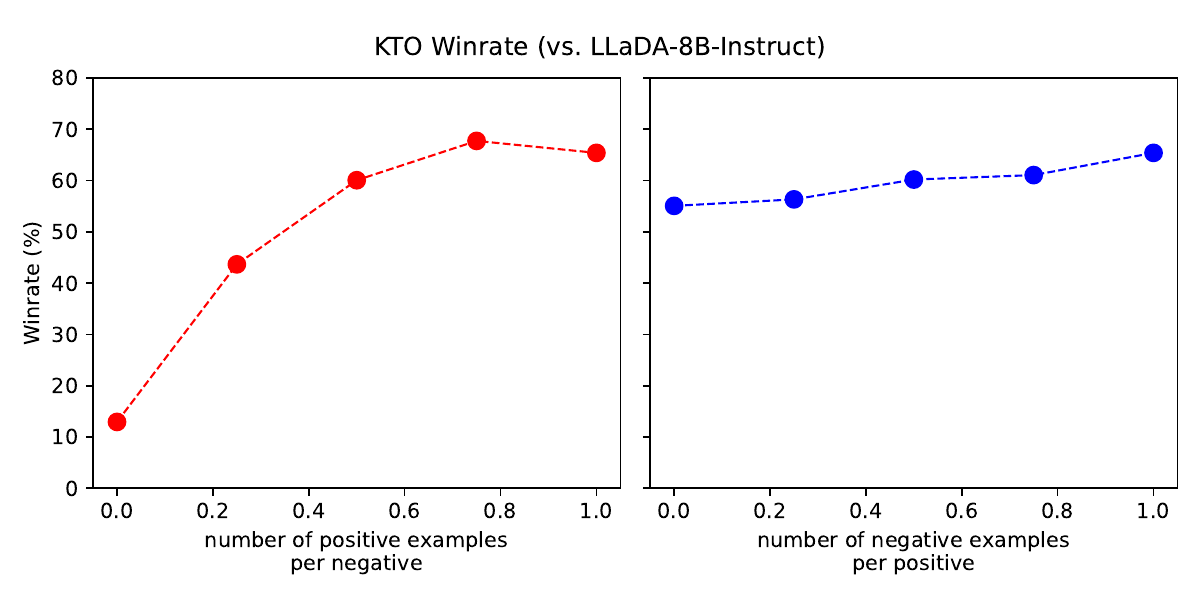} % use \textwidth
  \caption{Adjusted win rate vs.\ LLaDA-8B-Instruct on kto-mix-14k when varying the ratio of desirable to undesirable examples. Left: subsampling desirable examples; Right: subsampling undesirable examples. ELBO-KTO benefits more from desirable examples, consistent with gain sensitivity.}
  \label{fig:kto-imbalance}
\end{figure*}

Table~\ref{tab:winrates} reports adjusted win rates on the same test prompts for comparisons against the base model, LLaDA-8B-Instruct.
Our ELBO-KTO completions win in a clear majority of cases with an adjusted win rate of 65.9\% on kto-mix-14k and 62.3\textbf{\%} on UltraFeedback-Binary.

For context, we report results for the publicly released LLaDA-1.5 trained with VRPO. This is not an apples-to-apples comparison as the data quantity and distribution used to train LLaDa-1.5 differ from our method. LLaDA-1.5 was trained on $\sim350$k paired preferences from a different distribution. As shown in Table~\ref{tab:winrates}, our method shows better performance than LLaDA-1.5. This highlights that targeted preference data and unpaired alignment can be more sample-efficient than scaling paired data from a broader mix.

For further evaluation, we calculate the adjusted win rates for the chosen and the rejected targets from the test dataset.  We observe that the win rate of ELBO-KTO approaches the chosen-target win rate on kto-mix-14k and exceeds that of UltraFeedback-Binary.
The performance of rejected targets fall well below 50\%, confirming that the judge consistently disfavors undesirable responses.

\begin{table}[t]
  \centering
  \caption{Adjusted win rates (AWR) (\%) vs.\ LLaDA-8B-Instruct under FastChat \texttt{lm\_judge} evaluated using \texttt{gpt-4o-mini} on kto-mix-14k and UltraFeedback-Binary (UFB). We use $\beta=0.1$ for kto-mix-14k and $\beta=0.2$ for UFB.
  For both datasets, we calculate AWR for (i) our ELBO-KTO method, (ii) LLaDA-1.5, (iii) the dataset's chosen target, and (iv) the dataset's rejected target. Both output orderings are judged and ties are split equally.}
  \label{tab:winrates}
  \begin{tabular}{@{}lcc@{}}
    \toprule
    & kto-mix-14k & UFB \\
    \midrule
    ELBO-KTO  & \textbf{65.9} & \textbf{62.3} \\
    LLaDA-1.5 & 57.2 & 60.3 \\
    Chosen (label=True)    & 70.1          & 61.6 \\
    Rejected (label=False) & 47.3          & 40.0 \\
    \bottomrule
  \end{tabular}
\end{table}

\begin{table}[t]
\centering
\caption{Cross-generation judge comparison. Adjusted win rate (AWR, \%) with 90\% bootstrap CIs, majority-vote (MV), and Cohen’s $\kappa$.}
\label{tab:judge-robustness}
\begin{subtable}{\columnwidth}
\centering
\subcaption{kto-mix-14k}
\begin{tabular}{lcc}
\toprule
Judge & AWR (\%) & 90\% CI \\
\midrule
\texttt{gpt-4o-mini}  & 65.87 & [63.73, 68.00] \\
\texttt{gpt-4.1-mini} & 63.80 & [61.67, 66.07] \\
{MV (two judges)} & 63.07 & [61.20, 64.93] \\
\midrule
Cohen's $\kappa$ & 0.562 & [0.518, 0.604] \\
\bottomrule
\end{tabular}
\label{tab:judge-robustness-a}
\end{subtable}

\begin{subtable}{\columnwidth}
\centering
\vspace{2mm}
\subcaption{UltraFeedback-Binary}
\begin{tabular}{lcc}
\toprule
Judge & AWR (\%) & 90\% CI \\
\midrule
\texttt{gpt-4o-mini}  & 62.28 & [60.88, 63.73] \\
\texttt{gpt-4.1-mini} & 62.21 & [60.71, 63.68] \\
{MV (two judges)} & 61.38 & [60.11, 62.68] \\
\midrule
Cohen's $\kappa$ & 0.611 & [0.586, 0.636] \\
\bottomrule
\end{tabular}
\label{tab:judge-robustness-b}
\end{subtable}

\end{table}

We run an ablation comparing no baseline vs. the global per-batch baseline across three ($\beta$, learning-rate) settings under identical training and evaluation budgets. The global baseline consistently improves AWR by 5.27–9.34 percentage points (Table~\ref{tab:baseline-ablation}), indicating that centering margins with a stop-gradient batch mean stabilizes updates and yields better preference optimization.

To ensure our results are not an artifact of a particular LLM judge, we evaluated with two OpenAI models: \texttt{gpt-4o-mini} and \texttt{gpt-4.1-mini}, released about a year apart and trained under different alignment regimes. As shown in Table~\ref{tab:judge-robustness}, both judges report nearly the same adjusted win rates. The majority-vote rate sits close to the two judges, confirming consistency. Importantly, to measure the inter-judge agreement, we calculate Cohen's $\kappa$, a statistical measure of agreement between two annotators beyond chance. We obtain $\kappa=0.56$ on kto-mix-14k and $\kappa=0.61$ on UltraFeedback-Binary, indicating moderate-substantial agreement. These results demonstrate that our conclusions are robust to judge choice, even when the judges come from different generations of OpenAI models released one year apart.

\begin{table}[t]
\centering
\caption{Adjusted win rate (AWR, \%) with and without the global per-batch baseline across three $(\beta,\mathrm{lr})$ settings. Brackets show absolute gain (pp) vs.\ \emph{No baseline}.}
\label{tab:baseline-ablation}
\begin{tabular}{@{}cccc@{}}
\toprule
$\beta$ & LR & No baseline & Global mean \\
\midrule
0.2 & $5\times10^{-6}$ & 55.46 & \textbf{64.80} (+9.34) \\
0.1 & $1\times10^{-6}$ & 60.63 & \textbf{65.90} (+5.27) \\
0.2 & $1\times10^{-6}$ & 57.40 & \textbf{64.73} (+7.33) \\
\bottomrule
\end{tabular}
\end{table}

\subsection{KTO under class imbalance}
\label{sec:class-imbalance-results}

To test robustness under class imbalance, we subsample the kto-mix-14k dataset to vary the ratio of desirable to undesirable samples. We set $\lambda_D n_D = \lambda_U n_U$ to balance the effective contribution of each class in the loss, where $n_D$ and $n_U$ are the numbers of desired and undesired samples. Figure~\ref{fig:kto-imbalance} shows the adjusted win rate against the LLaDA-8B-Instruct model as we vary the class ratio.

Overall, we find that performance increases with more data for either class, consistent with KTO scaling more favorably when trained on larger sample sizes. However, the effect of imbalance is asymmetric. Reducing the number of desirable samples hurts performance much more severely than reducing the number of undesirable samples. This pattern mirrors the \emph{gain sensitivity} reported in ~\cite{ethayarajh2024ktomodelalignmentprospect}, where desirable examples contribute more strongly to alignment.

\subsection{Downstream generalization}
\label{sec:downstream-generaliztion-results}

We further evaluate ELBO-KTO by fine-tuning LLaDA-8B-Instruct on UltraFeedback-Binary and testing under identical decoding on {GSM8K}, {MMLU}, {HellaSwag}, {HumanEval}, and {GPQA}. As shown in Table~\ref{tab:downstream}, ELBO-KTO improves GSM8K (+3.26), slightly lifts MMLU (+0.58) and GPQA (+0.67), keeps HumanEval unchanged, and slightly drops HellaSwag ($-$0.75). Overall, ELBO-KTO preserves downstream performance while yielding modest gains on reasoning/knowledge benchmarks. Further evaluation details are provided in Appendix C.

\begin{table}
  \centering
  \caption{Performance comparison of LLaDA-8B-Instruct and ELBO-KTO on UltraFeedback-Binary on downstream tasks. The numbers in parentheses represent the number of shots used for evaluation.}
  \label{tab:downstream}
  \begin{tabular}{@{}lcc@{}}
    \toprule
    Task & LLaDA-8B-Instruct & ELBO-KTO \\
    \midrule
    GSM8K (5)     & 79.53 & \textbf{82.79} \\
    MMLU (5)      & 63.85 & \textbf{64.43} \\
    HellaSwag (0)  & \textbf{78.03} & 77.28 \\
    HumanEval (0)  & 42.68 & 42.68 \\
    GPQA (5)      & 29.02 & \textbf{29.69} \\
    \bottomrule
  \end{tabular}
\end{table}

\section{Related Work}

\paragraph{Image Diffusion Model Alignment}  
Preference optimization for diffusion models on image data has received significant attention. 
\citet{wallace2024diffusion} propose \emph{Diffusion-DPO}, which adapts DPO to diffusion models by defining a diffusion-based likelihood and directly optimizing over paired human preferences. 
\citet{hong2024margin} introduce a margin-aware preference optimization method that removes the need for a reference model, while \citet{lee2025calibrated} (CaPO) further improve alignment through reward calibration and Pareto-frontier-based pair selection. 
Similarly, \citet{li2024aligning} present \emph{Diffusion-KTO}, which aligns text-to-image diffusion models by maximizing expected human utility. 
Together, these approaches advance preference alignment in text-to-image diffusion models, whereas our proposed method specifically targets diffusion \emph{language} models.

\paragraph{Alignment for Diffusion Language Models}  
Diffusion language models are emerging as promising alternatives to autoregressive LLMs, including architectures such as DiffuLLaMA \citep{gong2024scaling}, Dream \citep{ye2025dream}, LLaDA \citep{nie2025largelanguagediffusionmodels}, and Mercury \citep{khanna2025mercury}.  
Recent work has begun exploring how to align dLLMs with preferences or reasoning objectives. The most closely related method is VRPO in LLaDA 1.5 \citep{zhu2025llada15variancereducedpreference}, which still operates under the paired preference regime and uses variance reduction to stabilize learning.  
Other diffusion-LM works emphasize reasoning or model structure rather than alignment:  
\citet{zhao2025d1} scale reasoning in dLLMs via reinforcement learning,  
\citet{tang2025wd1} apply weighted policy optimization for reasoning in dLLMs,  
\citet{yang2025mmada} explore masked diffusion models in code generation or multimodal settings (focusing on model performance rather than feedback alignment), and  
\citet{han2025discrete} propose trajectory-level alignment via stepwise decomposition of discrete diffusion processes.  
While these advance reasoning, architecture, or decoding, they generally do not support preference alignment from unpaired feedback. Our proposed ELBO-KTO method fills in this gap.

\section{CONCLUSION}

We introduced a simple and effective recipe for aligning diffusion LLMs with unpaired feedback by plugging ELBO surrogates into KTO’s value formulation: an ELBO margin for the policy–reference contrast and a mini-batch ELBO baseline, with VRPO-style variance reduction for stability. On LLaDA-8B-Instruct, the approach achieves strong automatic-judge gains on two test sets (65.9\% on kto-mix-14k, 62.3\% on UltraFeedback-Binary) and shows modest improvements on downstream tasks. Together, these results indicate that unpaired preference optimization is a viable path for diffusion LLM alignment and a practical complement to paired alignment methods.

Limitations include reliance on an ELBO surrogate (with potential bias/variance). Future work includes human evaluation, broader downstream tasks, principled choices for class weights $(\lambda_D,\lambda_U)$, stronger variance-reduction schemes, and combining unpaired ELBO-KTO with paired or reward-model objectives in dLLM settings.
% \section*{Acknowledgments}
% We thank Deepak Agarwal and Gungor Polatkan from LinkedIn Core AI for their support and guidance throughout this research. We are also grateful to Animesh Singh and Yanning Chen from the LinkedIn Training Platform team for their collaboration. We thank the \textsc{TravelPlanner} team at The Ohio State University~\citep{xie2024travelplanner} for releasing the benchmark and evaluation infrastructure that enabled this work. Finally, we acknowledge the close collaborations with \texttt{verl}, \texttt{sglang}, and the broader open-source communities, whose collective efforts, tools, and support have been instrumental in advancing agentic RL development.

\clearpage
% \bibliographystyle{iclr2026_conference}
% \bibliography{references}
\bibliographystyle{iclr2026_conference}

\appendix

\appendix
\thispagestyle{empty}

% \usepackage{amssymb}

% Supplementary material: To improve readability, you must use a single-column format for the supplementary material.
\onecolumn
% \aistatstitle{Supplementary Materials}

\section{ADDITIONAL FORMULATION}
\label{sec:appendix-a}
Diffusion Language Models (dLLMs) are inherently Masked Diffusion Models (MDMs), which incorporate a discrete random masking forward process and train a mask predictor to approximate the reverse unmasking process \citep{sahoo2024simpleeffectivemaskeddiffusion, ou2024your, austin2023structureddenoisingdiffusionmodels, lou2310discrete, nie2025largelanguagediffusionmodels}.

\subsection{Forward Process for Diffusion Language Models}
During the forward process of MDMs, an original sequence $y$ is progressively corrupted by masking tokens independently at a noise level $t \in [0,1]$. Let $x$ be the prompt, $y$ be the original response to the prompt, $y^i$ denote the $i-$th token of $y$, $L$ denote the total number of tokens in $y$, $K$ denote the vocabulary size, and $\mathbf{M}$ denote the mask token. The forward process $q$ to obtain the masked response $y_t$ at time $t$ can be formulated as
\begin{equation}
    q(y_t \mid y,x,t)  = \prod_{i=1}^{L}q(y_t^i \mid y^i, x, t),
\end{equation}

where

\begin{equation*}
    q(y_t^i \mid y^i, x, t) = 
\begin{cases}
    t, \qquad &y_t^i=\mathbf{M}, \\
    1-t,\qquad &y_t^i = y^i
\end{cases}.
\end{equation*}

\subsection{Reverse Process for Diffusion Language Models}
The reverse process starts at $t=1$ from a fully masked sequence and gradually unmasks tokens till $t=0$ to recover a fully unmasked language sequence. Let $p_\theta$ be the mask prediction models. Then, for timesteps $0 \le s < t \le 1$, the conditional distribution for the reverse process can be defined as
\begin{equation}
\label{eq:mdm-reverse}
q(y_s \mid s,t,y_t,x)
=\prod_{i=1}^L q\!\left(y_s^{\,i}\mid s,t,y_t,x\right)
\end{equation}
\begin{equation}
q\!\left(y_s^{\,i}\mid s,t,y_t,x\right)=
\begin{cases}
\displaystyle \frac{t-s}{t}\;p_\theta\!\left(y^{\,i}\mid y_t,x\right), & y_t^{\,i}=\mathbf{M}\ \wedge\ y_s^{\,i}\neq \mathbf{M},\\[8pt]
\displaystyle \frac{s}{t}, & y_t^{\,i}=\mathbf{M}\ \wedge\ y_s^{\,i}=\mathbf{M},\\[6pt]
1, & y_t^{\,i}\neq \mathbf{M}\ \wedge\ y_s^{\,i}=y_t^{\,i},\\[4pt]
0, & \text{otherwise.}
\end{cases}
\end{equation}

\subsection{Log-Likelihood ELBO}
The exact log-likelihood $\log \pi(y\mid x)$ intractable for dLLMs because of the nature of forward and reverse process. To tackle this, the log-likelihood is usually approximated by its ELBO \citep{lou2310discrete, ou2024your, nie2025largelanguagediffusionmodels}:
\begin{equation}
\label{eq:elbo-t}
\mathcal{B}_\pi(y\mid x) \triangleq
\mathbb{E}_{t\sim \mathcal{U}[0,1]}\;
\mathbb{E}_{y_t \sim q(y_t \mid t,y,x)}\,
\ell_\pi(y_t,t,y\mid x),
\end{equation}
where
\begin{equation}
\label{eq:ell-t}
\ell_\pi(y_t,t,y\mid x)
\;\triangleq\;
\left[
\frac{1}{t}\sum_{i=1}^L
\mathbf{1}\!\left[y_t^{\,i}=\mathbf{M}\right]\;
\log p_\theta\!\left(y^{\,i}\mid y_t,x\right)
\right]
\end{equation}
is the per-step loss of the mask prediction model.

Let $l$ be uniformly sampled from $\{1,2,\ldots,L\}$ and $y_l$ denote the sequence obtained by masking $l$ tokens without replacement. Using this terminology, 
\cite{zhu2025llada15variancereducedpreference, ou2024your} define another equivalent formulation for the ELBO approximation:
\begin{equation}
\label{eq:elbo-l}
\mathcal{B}'_\pi(y\mid x)
\;\triangleq\;
\mathbb{E}_{l \sim \mathcal{U}(\{1,2,\ldots,L\})}\;
\mathbb{E}_{y_l \sim q(y_l \mid l,y,x)}\,
\ell'_\pi(y_l,l,y\mid x),
\end{equation}
where
\begin{equation}
\label{eq:ell-l}
\ell'_\pi(y_l,l,y\mid x)
\;\triangleq\;
\left[
\frac{L}{l}\sum_{i=1}^L
\mathbf{1}\!\left[y_l^{\,i}=\mathbf{M}\right]\;
\log p_\theta\!\left(y^{\,i}\mid y_l,x\right)
\right].
\end{equation}
Following \cite{zhu2025llada15variancereducedpreference}, we adopt the $\mathcal{B}'_\pi(y\mid x)$ formulation for our experiments. Although $\mathcal{B}_\pi(y\mid x)$ and $\mathcal{B}'_\pi(y\mid x)$ are equivalent in expectation, the latter generally exhibits lower variance during estimation. This improvement arises because Eq.~\eqref{eq:elbo-l} deterministically masks exactly $l$ out of $L$ tokens in each sequence, producing more consistent samples, whereas Eq.~\eqref{eq:elbo-t} masks an expected fraction $t$ of tokens, introducing greater stochasticity. Consequently, we also employ the $\mathcal{B}'$ formulation in Eq.~\eqref{eq:elbo-l} as our log-likelihood estimator.

\section{THEORETICAL ANALYSIS AND PROOFS}
\label{sec:appendix-b}

%========================================
\subsection{Setup and Notation}
\label{sec:setup-notation}
%========================================

\paragraph{Dataset and Minibatch.}
Let the finite dataset be $\mathcal{D}_N=\{(x_n,y_n)\}_{n=1}^N$, where $x_n$ is prompt and $y_n$ is response.
A minibatch is $S=\{(x_i,y_i,s_i,\lambda_i)\}_{i=1}^m$, sampled uniformly. Here $s_i\in\{-1,+1\}$ denotes undesirable/desirable,
and $\lambda_i\in\{\lambda_D,\lambda_U\}$ are class weights with
$\lambda_{\max}=\max\{\lambda_D,\lambda_U\}$.

\paragraph{Policies and ELBO Plug-In.}
Let $\pi_\theta(\cdot\mid x)$ be the current policy and
$\pi_{\mathrm{ref}}(\cdot\mid x)$ be the frozen reference.
We use an ELBO surrogate $B_\pi(y\mid x)$ and an unbiased MC estimator
$\widehat B_\pi(y\mid x)$ satisfying
$\Emc[\widehat B_\pi(y\mid x)]=B_\pi(y\mid x)$.

\paragraph{ELBO Margins and Global Baseline.}
Define the ELBO margin as
\begin{equation}
    r_i \;:=\; B_{\pi_\theta}(y_i\mid x_i)-B_{\pi_{\mathrm{ref}}}(y_i\mid x_i),
\end{equation}
and its MC estimate as
\begin{equation}
    \hat r_i \;:=\; \widehat B_{\pi_\theta}(y_i\mid x_i)-\widehat B_{\pi_{\mathrm{ref}}}(y_i\mid x_i).
\end{equation}
The global mini-batch baseline is
$\hat b_0=\tfrac{1}{m}\sum_{j=1}^m \hat r_j$, and the signed centered margin is
$\hat\delta_i=s_i(\hat r_i-\hat b_0)$.
For gradients, we treat $\hat b_0$ and the reference term as stop-grad.

\paragraph{Objective and Target.}
With logistic link $g(u)=\sigma(\beta u)$, where $\sigma(x) = \frac{1}{1+e^{-x}}$, the per-item and batch losses are
\begin{align}
\hat\ell_i \;=\; \lambda_i\big(1-g(\hat\delta_i)\big),\qquad
\hat L(S) \;=\; \tfrac{1}{m}\sum_{i=1}^m \hat\ell_i.
\end{align}
The corresponding noise-free target \emph{global-baseline} replaces the MC estimates by expectations:
\begin{align}
L^{\mathrm{sg}}_{\mathrm{GB}}(S;\theta)
=\tfrac{1}{m}\sum_{i=1}^m \lambda_i\big(1-g\big(s_i(r_i-b_0)\big)\big),
\end{align}
where $b_0=\Emc[\hat b_0] = \tfrac{1}{m}\sum_{j=1}^m r_j$.

\paragraph{Lipschitz Constants.}
We use $L_g$ to denote Lipschitz constant of $g$ and $L_{g'}$ to denote the Lipschitz constant of $g'$. Since $g=\sigma(\beta u)$, the values of these constants come out to be $L_g = \beta/4$ and $L_{g'}=\beta^2/(6\sqrt{3})$.

\begin{lemma}[Lipschitz constants of the scaled sigmoid]
\label{lem:lipschitz-sigmoid}
Let $g(u)=\sigma(\beta u)$ with $\sigma(z)=\tfrac{1}{1+e^{-z}}$ and $\beta>0$.
Then $g$ is $L_g$-Lipschitz with $L_g=\beta/4$, and its derivative $g'$ is
$L_{g'}$-Lipschitz with $L_{g'}=\beta^2/(6\sqrt{3})$.
\end{lemma}
\begin{proof}
Since $g'(u)=\beta\,\sigma(\beta u)(1-\sigma(\beta u))$
and $\sigma(z)(1-\sigma(z))\le\tfrac14$ for all $z$,
we have $\sup_u|g'(u)|=\beta/4$,
so $g$ is $\beta/4$-Lipschitz.

For $g'$, differentiate once more:
\[
g''(u)
=\beta^2\,\sigma(\beta u)(1-\sigma(\beta u))(1-2\sigma(\beta u)).
\]
Let $p=\sigma(\beta u)\in(0,1)$ and define $h(p)=p(1-p)(1-2p)$.
The extrema of $|h(p)|$ satisfy $h'(p)=0$, where
$h'(p)=1-6p(1-p)=0\!\implies\! p_\star=\tfrac12(1\pm 1/\sqrt3)$.
At these points, $|h(p_\star)|=1/(6\sqrt3)$,
so $\sup_u|g''(u)|=\beta^2/(6\sqrt3)$.
By the mean-value theorem,
$|g'(u)-g'(v)|\le (\sup_t|g''(t)|)\,|u-v|$,
giving $L_{g'}=\beta^2/(6\sqrt3)$.
\end{proof}

\begin{lemma}[Centered-Margin Variance Aggregator]
\label{lem:delta-i-variance}
Fix a batch $S=\{(x_i,y_i,s_i,\lambda_i)\}_{i=1}^m$ and assume the MC design renders
$(\widehat r_1,\ldots,\widehat r_m)$ exchangeable when conditioning on $S$.
Let $v(S):=\Var_{\mathrm{MC}}(\widehat r_i)$ and $c(S):=\Cov_{\mathrm{MC}}(\widehat r_i,\widehat r_j)$ for $i\neq j$.
With $\widehat b_0=\tfrac{1}{m}\sum_{j=1}^m \widehat r_j$ and $\widehat\delta_i=s_i\big(\widehat r_i-\widehat b_0\big)$, we have
\[
\Var_{\mathrm{MC}}(\widehat\delta_i)
=\frac{m-1}{m}\big(v(S)-c(S)\big),
\]
and consequently,
\[
\Psi(S)
:=\tfrac{1}{m}\sum_{i=1}^m
\mathbb{E}_{\mathrm{MC}}\!\big[\Var_{\mathrm{MC}}(\widehat\delta_i)\big]
=\frac{m-1}{m}\big(v(S)-c(S)\big).
\]
\end{lemma}

\begin{proof}
Since $s_i^2=1$, $\Var_{\mathrm{MC}}(\widehat\delta_i)
=\Var_{\mathrm{MC}}(\widehat r_i-\widehat b_0)$.
Write
\[
\widehat r_i-\widehat b_0
=\Big(1-\tfrac{1}{m}\Big)\widehat r_i
-\tfrac{1}{m}\sum_{j\neq i}\widehat r_j
=:a\,\widehat r_i+b\sum_{j\neq i}\widehat r_j,
\]
with $a=\tfrac{m-1}{m}$ and $b=-\tfrac{1}{m}$.
By bilinearity of covariance,
\[
\Var_{\mathrm{MC}}(\widehat r_i-\widehat b_0)
= a^2\,\Var_{\mathrm{MC}}(\widehat r_i)
+ b^2\,\Var_{\mathrm{MC}}\!\Big(\sum_{j\neq i}\widehat r_j\Big)
+ 2ab\,\Cov_{\mathrm{MC}}\!\Big(\widehat r_i,\sum_{j\neq i}\widehat r_j\Big).
\]
Exchangeability of $(\widehat r_1,\ldots,\widehat r_m)$ implies
\[
\Var_{\mathrm{MC}}\!\Big(\sum_{j\neq i}\widehat r_j\Big)
=(m-1)v(S)+(m-1)(m-2)c(S),\qquad
\Cov_{\mathrm{MC}}\!\Big(\widehat r_i,\sum_{j\neq i}\widehat r_j\Big)
=(m-1)c(S).
\]
Substituting and simplifying with $a=\tfrac{m-1}{m}$ and $b=-\tfrac{1}{m}$ yields
\[
\Var_{\mathrm{MC}}(\widehat r_i-\widehat b_0)
=\frac{m-1}{m}\big(v(S)-c(S)\big).
\]

Conditioned on $S$, $\widehat\delta_i$ depends only on MC randomness,
so $\mathbb{E}_{\mathrm{MC}}\!\big[\Var_{\mathrm{MC}}(\widehat\delta_i)\big]
=\Var_{\mathrm{MC}}(\widehat\delta_i)$.
Since this value is identical for all $i$, averaging gives
\[
\Psi(S)
=\tfrac{1}{m}\sum_{i=1}^m \Var_{\mathrm{MC}}(\widehat\delta_i)
=\frac{m-1}{m}\big(v(S)-c(S)\big).
\]
\end{proof}

The result relies only on exchangeability
(common variance $v(S)$ and common covariance $c(S)$)
and centering by the batch mean $\widehat b_0$.
The signs $s_i\in\{-1,+1\}$ do not affect the variance
since $s_i^2=1$.

\subsection{Loss Bias}

\begin{restatedtheorem}{thm:loss-bias}
The minibatch loss bias relative to the global-baseline target satisfies
\begin{equation*}
\begin{aligned}
\Big|
\ED\Emc[\hat L(S)]
- \ED[L^{\mathrm{sg}}_{\mathrm{GB}}(S;\theta)]
\Big|
\;\le\;
\lambda_{\max} L_g\;
\ED\!\big[\sqrt{\Psi(S)}\big],
\end{aligned}
\end{equation*}
where $\Psi(S)$ is the centered-margin variance
aggregator defined in~\eqref{eq:Psi}. Here $|\cdot|$ refers to the absolute value.
\end{restatedtheorem}

\begin{proof}
We have
\[
\widehat L(S)=\tfrac{1}{m}\sum_{i=1}^m \lambda_i\big(1-g(\widehat\delta_i)\big),
\qquad
L^{\mathrm{sg}}_{\mathrm{GB}}(S;\theta)=\tfrac{1}{m}\sum_{i=1}^m \lambda_i\big(1-g(\delta_i^{\mathrm{gb}})\big),
\]
with $\widehat\delta_i=s_i(\widehat r_i-\widehat b_0)$, $\delta_i^{\mathrm{gb}}=s_i(r_i-b_0)$,
$\widehat b_0=\tfrac{1}{m}\sum_{j=1}^m \widehat r_j$, and $b_0=\tfrac{1}{m}\sum_{j=1}^m r_j$.
Since $L^{\mathrm{sg}}_{\mathrm{GB}}(S;\theta)$ is deterministic under MC randomness for fixed $S$,
\[
\Big|\ED\Emc[\widehat L(S)]-\ED[L^{\mathrm{sg}}_{\mathrm{GB}}(S;\theta)]\Big|
= \Big|\ED\big[\Emc[\widehat L(S)]-L^{\mathrm{sg}}_{\mathrm{GB}}(S;\theta)\big]\Big|
\;\le\; \ED\Big[\big|\Emc[\widehat L(S)]-L^{\mathrm{sg}}_{\mathrm{GB}}(S;\theta)\big|\Big],
\]
where the inequality uses $|\E[Z]|\le \E[|Z|]$ (triangle inequality for $|\cdot|$).
Expanding $\widehat L$ and $L^{\mathrm{sg}}_{\mathrm{GB}}$ and applying $|g(u)-g(v)|\le L_g|u-v|$ then yields
\[
\Big|\ED\Emc[\widehat L]-\ED[L^{\mathrm{sg}}_{\mathrm{GB}}]\Big|
\le L_g\;\ED\Emc\!\left[\frac{1}{m}\sum_{i=1}^m \lambda_i\,\big|\widehat\delta_i-\delta_i^{\mathrm{gb}}\big|\right]
\le \lambda_{\max}L_g\;\ED\Emc\!\left[\frac{1}{m}\sum_{i=1}^m \big|\widehat\delta_i-\delta_i^{\mathrm{gb}}\big|\right].
\]
Conditioned on $S$, $\Emc[\widehat r_i]=r_i$ and $\Emc[\widehat b_0]=b_0$, so $\Emc[\widehat\delta_i]=\delta_i^{\mathrm{gb}}$ and
$\big|\widehat\delta_i-\delta_i^{\mathrm{gb}}\big|=\big|\widehat\delta_i-\Emc[\widehat\delta_i]\big|$.
By Cauchy–Schwarz,
\[
\Emc\big|\widehat\delta_i-\Emc[\widehat\delta_i]\big|
\le \sqrt{\Emc\!\big[(\widehat\delta_i-\Emc[\widehat\delta_i])^2\big]}
= \sqrt{\Var_{\mathrm{MC}}(\widehat\delta_i)},
\]
since $\big(\Emc|X|\big)^2\le \Emc[X^2]\Emc[1]=\Emc[X^2]$ for $X=\widehat\delta_i-\Emc[\widehat\delta_i]$.
Therefore,
\[
\Big|\ED\Emc[\widehat L]-\ED[L^{\mathrm{sg}}_{\mathrm{GB}}]\Big|
\le \lambda_{\max}L_g\;\ED\!\left[\frac{1}{m}\sum_{i=1}^m \sqrt{\Var_{\mathrm{MC}}(\widehat\delta_i)}\right].
\]
Using concavity of $\sqrt{\cdot}$ for fixed $S$,
\[
\frac{1}{m}\sum_{i=1}^m \sqrt{\Var_{\mathrm{MC}}(\widehat\delta_i)}
\le \sqrt{\frac{1}{m}\sum_{i=1}^m \Var_{\mathrm{MC}}(\widehat\delta_i)},
\]
and noting that the variances are deterministic under MC conditioning, we identify
\[
\frac{1}{m}\sum_{i=1}^m \Var_{\mathrm{MC}}(\widehat\delta_i)
= \frac{1}{m}\sum_{i=1}^m \Emc\!\big[\Var_{\mathrm{MC}}(\widehat\delta_i)\big]
= \Psi(S).
\]
Taking $\ED$ completes the claim:
\[
\Big|\ED\Emc[\widehat L(S)]-\ED[L^{\mathrm{sg}}_{\mathrm{GB}}(S;\theta)]\Big|
\le \lambda_{\max} L_g\;\ED\!\big[\sqrt{\Psi(S)}\big].
\]
\end{proof}

\subsection{Loss Variance}

\begin{restatedtheorem}{thm:loss-var}
For the minibatch loss $\hat L(S)=\tfrac{1}{m}\sum_{i=1}^m \lambda_i\big(1-g(\hat\delta_i)\big)$
with $g(u)=\sigma(\beta u)$,
the MC-induced variance satisfies
\begin{align*}
\Var_{\mathrm{MC}}(\hat L(S))
\;\le\;
(\lambda_{\max}L_g)^2\;
\ED\!\big[\Psi(S)\big].
\end{align*}
\end{restatedtheorem}

\begin{proof}
Fix a batch $S$ and write $Z_i:=\lambda_i\,g(\widehat\delta_i)$ so that
\[
\widehat L(S)\;=\;\tfrac1m\sum_{i=1}^m \lambda_i\big(1-g(\widehat\delta_i)\big)
\qquad\Longrightarrow\qquad
\Var_{\mathrm{MC}}(\widehat L(S))=\Var_{\mathrm{MC}}\!\Big(\tfrac1m\sum_{i=1}^m Z_i\Big),
\]
since subtracting the constant $\tfrac1m\sum_i \lambda_i$ does not affect variance. Using $\Cov_{\mathrm{MC}}(Z_i,Z_j) \leq \sqrt{\Var_{\mathrm{MC}}(Z_i)\Var_{\mathrm{MC}}(Z_j)}$,
\begin{align*}
    \Var_{\mathrm{MC}}\!\Big(\tfrac1m\sum_{i=1}^m Z_i\Big)
&=\frac{1}{m^2}\Big(\sum_{i=1}^m \Var_{\mathrm{MC}}(Z_i)
+2\!\sum_{1\le i<j\le m}\!\Cov_{\mathrm{MC}}(Z_i,Z_j)\Big) \\
&\leq\frac{1}{m^2}\Big(\sum_{i=1}^m \Var_{\mathrm{MC}}(Z_i)
+2\!\sum_{1\le i<j\le m}\!\sqrt{\Var_{\mathrm{MC}}(Z_i)\Var_{\mathrm{MC}}(Z_j)}\Big) \\
&= \frac{1}{m^2}\Big(\sum_{i=1}^m \sqrt{\Var_{\mathrm{MC}}(Z_i)}\Big)^{\!2}.
\end{align*}

Applying Cauchy–Schwarz again in the form $(\sum_i a_i)^2\le m\sum_i a_i^2$ with
$a_i=\sqrt{\Var_{\mathrm{MC}}(Z_i)}$ gives
\[
\Var_{\mathrm{MC}}\!\Big(\tfrac1m\sum_{i=1}^m Z_i\Big)
\;\le\; \frac{1}{m}\sum_{i=1}^m \Var_{\mathrm{MC}}(Z_i).
\]
Next, use Lipschitz contraction for each $i$:
since $u\mapsto \lambda_i g(u)$ is $(\lambda_i L_g)$-Lipschitz,
\[
\Var_{\mathrm{MC}}(Z_i)
=\Var_{\mathrm{MC}}(\lambda_i g(\widehat\delta_i))
\;\le\;(\lambda_i L_g)^2\,\Var_{\mathrm{MC}}(\widehat\delta_i),
\]
where we used $\Var(f(X))=\min_a \Emc[(f(X)-a)^2]
\le \Emc[(f(X)-f(\Emc X))^2]\le L^2 \Emc[(X-\Emc X)^2]=L^2\Var(X)$ for an $L$-Lipschitz $f$.
Hence,
\[
\Var_{\mathrm{MC}}(\widehat L(S))
\;\le\; \frac{(\lambda_{\max}L_g)^2}{m}\sum_{i=1}^m \Var_{\mathrm{MC}}(\widehat\delta_i)
\;=\;(\lambda_{\max}L_g)^2\,\Psi(S),
\]
using the definition (and MC-determinism) of $\Psi(S)=\tfrac{1}{m}\sum_{i=1}^m \Var_{\mathrm{MC}}(\widehat\delta_i)$.
Finally, taking $\ED[\cdot]$ over batches yields
\[
\ED\big[\Var_{\mathrm{MC}}(\widehat L(S))\big]
\;\le\;(\lambda_{\max}L_g)^2\,\ED[\Psi(S)].
\]
\end{proof}

\subsection{Gradient Bias}

\begin{restatedtheorem}{thm:grad-bias}
Let $\|\cdot\|$ be the $L_2$-norm.
If $\|\nabla_\theta r_i\|\le \bar G$ for all items, then
\begin{equation*}
\begin{aligned}
\bigl\|\ED&\Emc[\hat G(S)] - \ED[G(S)]\bigr\|
\le \\ &\lambda_{\max} L_{g'}\, \ED\!\big[\sqrt{\Psi(S)}\big]\, \bar G + \lambda_{\max} L_g\, \ED\!\big[\sqrt{\bar v_{\nabla}(S)}\big],
\end{aligned}
\end{equation*}

where $\bar v_{\nabla}(S) = \tfrac{1}{m}\sum_{i=1}^m
\Emc\!\big[\|\nabla_\theta \hat r_i - \nabla_\theta r_i\|^2\big]$, $\hat G(S)$ is the stochastic gradient for the minibatch $S$ as defined in \eqref{eq:g-hat-s}, and $G$ is defined in \eqref{eq:g}.
\end{restatedtheorem}

\begin{proof}
We start by expanding the difference of the expectations and splitting it into two parts:
\begin{align}
\ED\EMC[\hat G(S)] - \ED[G(S)]
&=\frac{1}{m}\sum_{i=1}^m
\ED\EMC\!\big[\hat a_i \nabla_\theta \hat r_i - a_i \nabla_\theta r_i\big]\nonumber\\[-2pt]
&=\frac{1}{m}\sum_{i=1}^m
\ED\!\Big(
\underbrace{\EMC[\hat a_i-a_i]}_{\text{(A) weight noise}}\;\nabla_\theta r_i
\;+\;
\underbrace{\EMC\!\big[\hat a_i(\nabla_\theta \hat r_i-\nabla_\theta r_i)\big]}_{\text{(B) score-gradient noise}}
\Big).
\label{eq:gb-decomp-grad-bias}
\end{align}

\paragraph{(A) Weight-noise term.}\label{ref:grad-bias-a}
Using $a_i=-\lambda_i s_i g'(\delta_i)$ and $\hat a_i=-\lambda_i s_i g'(\hat\delta_i)$ with $g'$ $L_{g'}$-Lipschitz,
\begin{align*}
\big|\EMC[\hat a_i-a_i]\big|
&\le \lambda_i\,\EMC\!\big[\,|g'(\hat\delta_i)-g'(\delta_i)|\,\big]
\le \lambda_i L_{g'}\,\EMC\!\big[\,|\hat\delta_i-\delta_i|\,\big]\nonumber\\
&\le \lambda_i L_{g'}\sqrt{\EMC\!\big[(\hat\delta_i-\delta_i)^2\big]}
= \lambda_i L_{g'}\sqrt{\EMC\!\big[\Var_{\mathrm{MC}}(\hat\delta_i)\big]},
\end{align*}
where we used Cauchy–Schwarz $\big(\EMC\big[|\hat\delta_i-\delta_i|.|1|\big] \leq \sqrt{\EMC\big[(\hat\delta_i-\delta_i)^2\big]\EMC\big[1^2\big]} \big)$ and $\EMC[\hat\delta_i]=\delta_i$. Hence,
\begin{align}
\frac{1}{m}\sum_{i=1}^m
\ED\!\Big(
\EMC[\hat a_i-a_i]\;\nabla_\theta r_i\big) &\leq\frac{1}{m}\sum_{i=1}^m \ED\!\big[\,|\EMC[\hat a_i-a_i]|\ \|\nabla_\theta r_i\|\,\big]\nonumber \\
&\le \frac{\lambda_{\max} L_{g'}}{m}\sum_{i=1}^m
\ED\!\Big[\sqrt{\EMC\Var_{\mathrm{MC}}(\hat\delta_i)}\,\|\nabla_\theta r_i\|\Big]\nonumber\\
&\le \lambda_{\max} L_{g'}\,\bar G\,
\ED\!\Big[\frac{1}{m}\sum_{i=1}^m \sqrt{\EMC\Var_{\mathrm{MC}}(\hat\delta_i)}\Big]\nonumber\\
&\le \lambda_{\max} L_{g'}\,\bar G\,\ED\!\big[\sqrt{\Psi(S)}\big],
\label{eq:termA-grad-bias}
\end{align}

using the concavity of $x\mapsto\sqrt{x}$, and where we (re)use
$\Psi(S):=\tfrac{1}{m}\sum_{i=1}^m \EMC\Var(\hat\delta_i)$.

\paragraph{(B) Score-gradient noise term.}
By Cauchy–Schwarz and the uniform bound $|g'(u)|\le L_g$,
\begin{align*}
\big\|\EMC[\hat a_i(\nabla_\theta \hat r_i-\nabla_\theta r_i)]\big\|
&\le \sqrt{\EMC[\hat a_i^2]}\;
\sqrt{\EMC\!\big[\|\nabla_\theta \hat r_i-\nabla_\theta r_i\|^2\big]}\nonumber\\
&\le \lambda_{\max} L_g\;
\sqrt{\EMC\!\big[\|\nabla_\theta \hat r_i-\nabla_\theta r_i\|^2\big]}.
\end{align*}
Averaging over $i$ and applying the same concavity argument gives
\begin{align}
\frac{1}{m}\sum_{i=1}^m \ED\!\Big[
\big\|\EMC[\hat a_i(\nabla_\theta \hat r_i-\nabla_\theta r_i)]\big\|
\Big]
\;\le\;
\lambda_{\max} L_g\;\ED\!\big[\sqrt{\bar v_{\nabla}(S)}\big],
\label{eq:termB-grad-bias}
\end{align}
with $\displaystyle \bar v_{\nabla}(S)=\tfrac{1}{m}\sum_{i=1}^m
\EMC\!\big[\|\nabla_\theta \hat r_i-\nabla_\theta r_i\|^2\big]$.

\paragraph{Conclusion.}
Combining \eqref{eq:termA-grad-bias} and \eqref{eq:termB-grad-bias} in \eqref{eq:gb-decomp-grad-bias} and applying the triangle inequality yields
\[
\Big\|\ED\EMC[\hat G(S)] - \ED[G(S)]\Big\|
\le
\lambda_{\max} L_{g'}\, \ED\!\big[\sqrt{\Psi(S)}\big]\, \bar G
\;+\;
\lambda_{\max} L_g\, \ED\!\big[\sqrt{\bar v_{\nabla}(S)}\big],
\]
which is the desired bound.
\end{proof}

This result shows that gradient bias has two contributions:
weight noise and score-gradient noise. Weight noise arises from applying $g$ to noisy centered margin $\hat \delta_i$, and is controlled by
$\Psi(S)$. Score-gradient noise arises from stochastic ELBO gradients
controlled by $\bar v_{\nabla}(S)$, which can be reduced by increasing the MC budget.

\subsection{Gradient Variance}

Let
\begin{equation*}
    \begin{gathered}
        \tilde G^2(S) = \tfrac{1}{m}\sum_{i=1}^m \Emc\!\big[\|\nabla_\theta \hat r_i\|^2\big], \\
\bar c_{\nabla}(S)=\tfrac{1}{m(m-1)}\sum_{i\ne j}\Emc[\langle \xi_i,\xi_j\rangle],
    \end{gathered}
\end{equation*}
where $\xi_i = \nabla_\theta \hat r_i - \nabla_\theta r_i$, and assume unbiased ELBO gradients $\Emc[\nabla_\theta \hat r_i]=\nabla_\theta r_i$.

\begin{restatedtheorem}{thm:grad-var}
For mini-batch $S$, let $U(S)=\frac{1}{m}\sum_{i=1}^m a_i\,\nabla_\theta \hat r_i$.
Then
\begin{equation*}
\begin{aligned}
&\Var_{\mathrm{MC}}\!\big(\hat G(S)\big)
\le\; \\
&\qquad \Big(\sqrt{\Var_{\mathrm{MC}}\!\big(U(S)\big)}
\;+\;
\lambda_{\max} L_{g'}\ \sqrt{\Psi(S)}\; \tilde G(S)\Big)^{\!2},
\end{aligned}
\end{equation*} 
where 
\begin{equation*}
\begin{aligned}
\Var_{\mathrm{MC}}(U)
\;\le\;
(\lambda_{\max}L_g)^2
\Big(\frac{\bar v_{\nabla}(S)}{m}
\;+\;\frac{m-1}{m}\ \bar c_{\nabla}(S)\Big),
\end{aligned}
\end{equation*}
and $\hat G(S):=\frac{1}{m}\sum_{i=1}^m \hat a_i \,\nabla_\theta \hat r_i$ is the stochastic gradient of the mini-batch $S$.

\end{restatedtheorem}

\begin{proof}
Let
\[
U(S):=\frac{1}{m}\sum_{i=1}^m a_i\,\nabla_\theta \hat r_i,\qquad
V(S):=\frac{1}{m}\sum_{i=1}^m (\hat a_i-a_i)\,\nabla_\theta \hat r_i,
\]
so that $\hat G=U+V$.
We also use $\xi_i:=\nabla_\theta \hat r_i-\nabla_\theta r_i$, so $\Emc[\xi_i]=0$.

\paragraph{(A) Relating $\Var_{\mathrm{MC}}(\hat G)$ to $\Var_{\mathrm{MC}}(U)$ and $\Var_{\mathrm{MC}}(V)$.}
By definition of vector variance,
\[
\Var_{\mathrm{MC}}(\hat G)
=\Emc\!\left[\left\|\hat G-\Emc[\hat G]\right\|_2^2\right].
\]
Since $\hat G=U+V$, write
\[
\hat G-\Emc[\hat G]=(U-\Emc[U])+(V-\Emc[V])=:A+B.
\]
Using the pointwise triangle inequality $\|A+B\|_2\le \|A\|_2+\|B\|_2$ and then squaring and taking expectation, we get
\[
\Emc\!\left[\|A+B\|_2^2\right]
\le \Emc\!\left[(\|A\|_2+\|B\|_2)^2\right]
= \Emc\!\left[\|A\|_2^2\right] + 2\,\Emc\!\left[\|A\|_2\|B\|_2\right]+\Emc\!\left[\|B\|_2^2\right].
\]
Apply Cauchy--Schwarz to the cross term:
\[
\Emc\!\left[\|A\|_2\|B\|_2\right]\ \le\ \sqrt{\Emc[\|A\|_2^2]}\ \sqrt{\Emc[\|B\|_2^2]}.
\]
Combining,
\[
\Emc\!\left[\|A+B\|_2^2\right]
\le \Big(\sqrt{\Emc[\|A\|_2^2]}+\sqrt{\Emc[\|B\|_2^2]}\Big)^2.
\]
Recognizing $\Emc[\|A\|_2^2]=\Var_{\mathrm{MC}}(U)$ and $\Emc[\|B\|_2^2]=\Var_{\mathrm{MC}}(V)$, we conclude
\begin{equation}\label{eq:sd-sum-derived}
\sqrt{\Var_{\mathrm{MC}}(\hat G)}\ \le\ \sqrt{\Var_{\mathrm{MC}}(U)}+\sqrt{\Var_{\mathrm{MC}}(V)}.
\end{equation}

\paragraph{(B) Bounding $\Var_{\mathrm{MC}}(V)$.}
By the Lipschitz property of $g'$,
\[
|\hat a_i-a_i|
=\lambda_i\,|g'(\hat\delta_i)-g'(\delta_i)|
\le \lambda_i L_{g'}\,|\hat\delta_i-\delta_i|.
\]
Therefore,
\begin{align*}
\Var_{\mathrm{MC}}(V)
&=\Var_{\mathrm{MC}}\!\Big(\frac{1}{m}\sum_i (\hat a_i-a_i)\,\nabla_\theta \hat r_i\Big)
\ \le\ \Emc\Big\|\frac{1}{m}\sum_i (\hat a_i-a_i)\,\nabla_\theta \hat r_i\Big\|_2^2\\
&=\frac{1}{m^2}\sum_{i,j}\Emc\!\big[(\hat a_i-a_i)(\hat a_j-a_j)\,\langle \nabla_\theta \hat r_i,\nabla_\theta \hat r_j\rangle\big]\\
&\le \frac{1}{m^2}\Big(\sum_i \sqrt{\Emc\!\big[(\hat a_i-a_i)^2\,\|\nabla_\theta \hat r_i\|_2^2\big]}\Big)^2
\quad\text{(Cauchy--Schwarz across $i$)}\\
&\le \frac{1}{m^2}\Big(\sum_i \sqrt{\Emc[(\hat a_i-a_i)^2]}\ \sqrt{\Emc\|\nabla_\theta \hat r_i\|_2^2}\Big)^2
\quad\text{(Cauchy--Schwarz in expectation)}\\
&\le \lambda_{\max}^2 L_{g'}^2
\Big(\frac{1}{m}\sum_i \sqrt{\Emc[(\hat\delta_i-\delta_i)^2]}\ \sqrt{\Emc\|\nabla_\theta \hat r_i\|_2^2}\Big)^2\\
&\le \lambda_{\max}^2 L_{g'}^2
\Big(\frac{1}{m}\sum_i \Emc[(\hat\delta_i-\delta_i)^2]\Big)
\Big(\frac{1}{m}\sum_i \Emc\|\nabla_\theta \hat r_i\|_2^2\Big)
\quad\text{(Cauchy--Schwarz on the index $i$)}\\
&= \lambda_{\max}^2 L_{g'}^2\ \Psi(S)\ \tilde G^2(S).
\end{align*}
Taking square roots yields
\begin{equation}\label{eq:V-sd}
\sqrt{\Var_{\mathrm{MC}}(V)}\ \le\ \lambda_{\max} L_{g'}\ \sqrt{\Psi(S)}\ \tilde G(S).
\end{equation}

\paragraph{(C) Bounding $\Var_{\mathrm{MC}}(U)$.}
Since $\Emc[\xi_i]=0$ and $U-\Emc[U]=\tfrac{1}{m}\sum_i a_i\,\xi_i$ with deterministic $a_i$,
\begin{equation}\label{eq:U-exact}
\Var_{\mathrm{MC}}(U)
=\Emc\Big\|\frac{1}{m}\sum_{i=1}^m a_i\,\xi_i\Big\|_2^2
=\frac{1}{m^2}\sum_{i,j=1}^m a_i a_j\,\Emc\big[\langle \xi_i,\xi_j\rangle\big].
\end{equation}
Using $|a_i|\le \lambda_{\max}|g'(\delta_i)|\le \lambda_{\max}L_g$ and splitting diagonal/off-diagonal terms,
\begin{align*}
\Var_{\mathrm{MC}}(U)
&\le \frac{(\lambda_{\max}L_g)^2}{m^2}
\bigg(\sum_{i=1}^m \Emc\|\xi_i\|_2^2+\sum_{i\ne j}\Emc\langle \xi_i,\xi_j\rangle\bigg)\\
&=(\lambda_{\max}L_g)^2
\bigg(\frac{1}{m}\cdot\frac{1}{m}\sum_{i=1}^m \Emc\|\xi_i\|_2^2
+\frac{m-1}{m}\cdot\frac{1}{m(m-1)}\sum_{i\ne j}\Emc\langle \xi_i,\xi_j\rangle\bigg)\\
&=(\lambda_{\max}L_g)^2
\Big(\frac{\bar v_{\nabla}(S)}{m}+\frac{m-1}{m}\ \bar c_{\nabla}(S)\Big),
\end{align*}
which is the claimed proxy bound.

\medskip
\textbf{ (D) Combine.}
From \eqref{eq:sd-sum-derived} and \eqref{eq:V-sd},
\[
\sqrt{\Var_{\mathrm{MC}}(\hat G(S))}
\ \le\
\sqrt{\Var_{\mathrm{MC}}(U(S))}\;+\;\lambda_{\max} L_{g'}\ \sqrt{\Psi(S)}\ \tilde G(S).
\]
Squaring both sides gives
\[
\Var_{\mathrm{MC}}(\hat G(S))
\ \le\
\Big(\sqrt{\Var_{\mathrm{MC}}(U(S))}+\lambda_{\max} L_{g'}\,\sqrt{\Psi(S)}\,\tilde G(S)\Big)^{\!2}.
\]
Together with the bound on $\Var_{\mathrm{MC}}(U)$ above, this establishes the theorem.
\end{proof}

For the special case of independent per–item MC seeds, 
we can assume that the per–item MC randomness is conditionally independent across items in the batch, so that $\Emc[\langle \xi_i,\xi_j\rangle]=0$ for $i\neq j$.
Then the variance of $U(S)$ collapses to the diagonal:
\begin{equation*}
\begin{aligned}
\Var_{\mathrm{MC}}\!\big(U(S)\big)
= \frac{1}{m^2}\sum_{i=1}^m a_i^2\,\Emc\!\big[\|\xi_i\|^2\big]
\;\le\;
(\lambda_{\max}L_g)^2\,\frac{\bar v_{\nabla}(S)}{m},
\end{aligned}
\end{equation*}
where $\bar v_{\nabla}(S)=\tfrac{1}{m}\sum_{i=1}^m \Emc\!\big[\|\xi_i\|^2\big]$.
Substituting into Theorem~\ref{thm:grad-var} gives
\begin{equation*}
\begin{aligned}
\Var_{\mathrm{MC}}\!\big(\hat G(S)\big)
\le
\Big((\lambda_{\max}L_g)\,\sqrt{\tfrac{\bar v_{\nabla}(S)}{m}}
\;+\;
\lambda_{\max} L_{g'}\ \sqrt{\Psi(S)}\;\tilde G(S)\Big)^{\!2}.
\end{aligned}
\end{equation*}

\subsection{Global Baseline Optimality}

\begin{restatedlemma}{lem:psi-diff}
For any baseline $b$ that is constant across items in a batch and may depend on MC randomness,
\[
\Psi_b(S)-\Psi_{\hat b_0}(S)
\;=\;
\Var_{\mathrm{MC}}\!\big(b-\hat b_0\big)
\;\ge\;0,
\]
where $\hat b_0=\tfrac{1}{m}\sum_j \hat r_j$. Hence every baseline of the form
$b=\hat b_0+K$ with deterministic constant $K$ attains the same minimum value
$\Psi_b(S)=\Psi_{\hat b_0}(S)$. If, in addition, $\Emc[b]=b_0$, then $K=0$
and the unique minimizer is $b=\hat b_0$.
\end{restatedlemma}

\begin{proof}
Let $\Delta:=b-\hat b_0$. Using $\Var(X-Y)=\Var X+\Var Y-2\Cov(X,Y)$ and $\Psi(S)$ from Lemma~\ref{lem:delta-i-variance},
\[
\Psi_b-\Psi_{\hat b_0}
=\frac{1}{m}\sum_i \Emc\!\big[\Var((\hat r_i-\hat b_0)-\Delta)\big]
-\frac{1}{m}\sum_i \Emc\!\big[\Var(\hat r_i-\hat b_0)\big].
\]
This equals
\[
\Var(\Delta)
-\frac{2}{m}\sum_{i=1}^m \Cov(\hat r_i-\hat b_0,\Delta)
=\Var(\Delta)-2\,\Cov(\hat b_0-\hat b_0,\Delta)
=\Var(\Delta)\ge 0,
\]
since $\frac{1}{m}\sum_i\Cov(\hat r_i,\Delta)=\Cov(\hat b_0,\Delta)$.
Equality holds iff $\Var(\Delta)=0$, i.e.\ $b-\hat b_0$ is a constant.
If moreover $\Emc[b]=b_0$, then that constant must be $0$, so $b=\hat b_0$.
\end{proof}

This result says that $\hat b_0$ is variance-optimal for all possible values of $b$. It requires no additional compute and is thus chosen as the principled default for ELBO-KTO.

\section{EXPERIMENTAL DETAILS}
\label{sec:appendix-c}
This section specifies datasets, models, preprocessing, optimization, compute, decoding, and evaluation protocols used in our experiments.

%----------------------------------------
\subsection{Models and Tokenization}
\label{sec:models}
%----------------------------------------
We fine-tune LLaDA-8B-Instruct as the policy and use a frozen copy of the same checkpoint as the reference $\pi_{\text{ref}}$ for ELBO-KTO. We apply the official chat template during both training and inference for consistency with the base model.\footnote{\url{https://huggingface.co/GSAI-ML/LLaDA-8B-Instruct}}

%----------------------------------------
\subsection{Data and Preprocessing}
\label{sec:data}
%----------------------------------------
\paragraph{Preference Data.}
We use kto-mix-14k and UltraFeedback-Binary for training. Each example is $(x,y,s)$ with $s\in\{+1,-1\}$ indicating desirable/undesirable. For UltraFeedback-Binary, we convert pairs to unpaired labels by taking the \emph{chosen} response as desirable and the \emph{rejected} response as undesirable. Unless noted, we train on the official train split and report on the test split.

\paragraph{Length and Formatting.}
We cap the concatenated prompt+response length at $L{=}4096$ tokens, pad with \texttt{|EOS|} to length 4096, and \emph{only} mask completion tokens during training. Because we fine-tune from LLaDA-8B-Instruct, the chat template is applied at train and inference time.

\paragraph{Licenses.}
Dataset licenses used in the paper are summarized in Table~\ref{tab:dataset-licenses}.

\begin{table}[t]
\centering
\caption{Licenses for datasets used in this work.}
\label{tab:dataset-licenses}
\begin{tabular}{@{}l l@{}}
\toprule
\textbf{Dataset} & \textbf{License}  \\
\midrule
kto-mix-14k & MIT \\
UltraFeedback-Binary & MIT\\
GSM8K & MIT  \\
HumanEval & MIT \\
MMLU & MIT \\
GPQA & CC-BY-4.0  \\
HellaSwag & MIT \\
\bottomrule
\end{tabular}
\end{table}

%----------------------------------------
\subsection{Training}
\label{sec:training}
%----------------------------------------

\paragraph{Optimizer and Schedule.}
We train for one epoch with batch size $8$ using AdamW (weight decay $0.01$, $\beta_1{=}0.9$, $\beta_2{=}0.95$), a $3\%$ linear warmup, and cosine decay.

\paragraph{Learning Rate and MC Budget.}
For kto-mix-14k: peak learning-rate $1{\times}10^{-6}$ with $8$ MC samples per example. 
For UltraFeedback-Binary (UFB): peak learning-rate $5{\times}10^{-7}$ with $4$ MC samples per example. The same UFB-trained checkpoint is used for downstream generalization experiments in Section~\ref{sec:downstream-generaliztion-results}.

\paragraph{Hyperparameters.} For the results reported in Tables~\ref{tab:winrates}-\ref{tab:baseline-ablation} we set $\lambda_D=\lambda_U=1$. For results in Tables~\ref{tab:winrates}-\ref{tab:judge-robustness}, we set $\beta=0.1$ for kto-mix-14k and $\beta=0.2$ for UltraFeedback-Binary. For the ablation study in Table~\ref{tab:baseline-ablation}, we use kto-mix-14k and report the values of $\beta$ and learning-rate in the table. For downstream results in Section~\ref{sec:downstream-generaliztion-results}, we use the ELBO-KTO model trained on UltraFeedback-Binary used in Section~\ref{sec:overall-results}.

\paragraph{Reference Precompute.}
For memory efficiency during training we precompute $\widehat{B}_{\pi_{\text{ref}}}$ once and cache the mask metadata, so we do not keep both the policy and reference models resident simultaneously. This choice does not alter the ELBO-KTO objective or estimator; it only reduces memory/compute overhead. Implementations with sufficient memory can compute $\widehat{B}_{\pi_{\text{ref}}}$ on-the-fly by holding both models at once and will obtain identical updates (we share the same masks via cached metadata either way). 

\paragraph{Compute.}
Unless specified, experiments run on $8\times$ NVIDIA H100 80GB with FSDP (no offload) and per-GPU microbatch $=1$ (global batch $=8$). Typical wall-clock: precompute $\widehat{B}_{\pi_{\text{ref}}}$ on ${\sim}13.5$k samples $\approx 37$ min; training after precompute $\approx 172$ min. We include precompute timings only for reproducibility/accounting, not as a core component of the method.

%----------------------------------------
\subsection{Decoding and Evaluation Protocol}
\label{sec:decoding}
%----------------------------------------

Unless noted otherwise, we use temperature $0.0$, classifier-free guidance $0.0$, and low-confidence remasking strategy for all evaluation purposes as recommended by official LLaDA evaluation script\footnote{\href{https://github.com/ML-GSAI/LLaDA}{https://github.com/ML-GSAI/LLaDA}}. For Sections~\ref{sec:overall-results}-~\ref{sec:class-imbalance-results}, we use generation length $512$, block length $32$, $512$ diffusion steps. For evaluation on downstream tasks in Section~\ref{sec:downstream-generaliztion-results}, task-specific parameters are in Table~\ref{tab:eval-config}. We evaluate LLaDA-based instruction-tuned models with conditional generation, following \cite{nie2025largelanguagediffusionmodels}. Because their public repo lacks scripts for conditional evaluation of LLaDA-8B-Instruct, we implement an evaluation script on top of \texttt{lm-evaluation-harness}\footnote{\href{https://github.com/EleutherAI/lm-evaluation-harness}{https://github.com/EleutherAI/lm-evaluation-harness}}.

\begin{table}[h]
  \centering
  \caption{Evaluation configuration per task (gen length, block length, diffusion steps, \#few-shot).}
  \label{tab:eval-config}
  \begin{tabular}{@{}lcccc@{}}
    \toprule
    Task & Gen Len & Block Len & Steps & Few-shot \\
    \midrule
    GSM8K & 256 & 8 & 256 & 5 \\
    MMLU & 3 & 3 & 3 & 5 \\
    HellaSwag & 3 & 3 & 3 & 0 \\
    HumanEval & 512 & 512 & 32 & 0 \\
    GPQA & 128 & 128 & 64 & 5 \\
    \bottomrule
  \end{tabular}
\end{table}

%----------------------------------------
\subsection{Win-Rate, Majority-Vote and Cohen's $\kappa$}
\label{sec:judge}
%----------------------------------------
\paragraph{Win-rate protocol.}
We report \emph{Adjusted Win Rate} (AWR) of the tuned model versus the base model using an open-source LLM-as-a-judge (FastChat \texttt{lm\_judge}).\footnote{\href{https://github.com/lm-sys/FastChat}{https://github.com/lm-sys/FastChat}}
For each prompt $x$, we generate two completions under identical decoding: $y^{(A)}$ from the tuned model and $y^{(B)}$ from the base model.
The judge receives $(x, y^{(A)}, y^{(B)})$ with a standardized instruction template and returns a label in $\{\text{A wins}, \text{B wins}, \text{tie}\}$.
To mitigate position bias, we evaluate \emph{both} orderings, $(x, y^{(A)}, y^{(B)})$ and $(x, y^{(B)}, y^{(A)})$, and declare a win for a model only if \emph{both} orderings favor the same model; otherwise the outcome is counted as a tie.
We then compute
\[
\mathrm{AWR}
\;=\;
\frac{\#\text{A wins} + 0.5\,\#\text{ties}}{\#\text{A wins} + \#\text{B wins} + \#\text{ties}}\,,
\]
and report the mean AWR over the full test set.

\paragraph{Two-judge majority vote and agreement.}
We also report AWR under a majority vote of two independent judges.
The same two-ordering rule is applied per judge; the ensemble outcome is a win for a model only if \emph{both} judges, under \emph{both} orderings, favor that model; otherwise it is treated as a tie.
Inter-judge agreement is measured with Cohen’s $\kappa$, treating win, loss, and tie as three distinct classes.

\paragraph{Uncertainty.}
We report $90\%$ confidence intervals via nonparametric bootstrapping with $5{,}000$ resamples of the test set (sampling prompts with replacement and recomputing metrics per resample); intervals are taken from the $5^\text{th}$ and $95^\text{th}$ percentiles.

\end{document}